\documentclass[letterpaper]{article} %
\usepackage{aaai2027}  %
\usepackage[hyphens]{url}  %
\usepackage{graphicx} %
\usepackage{natbib}  %
\usepackage{caption} %
\usepackage{algorithm}
\usepackage{algorithmic}

\usepackage{booktabs}

\usepackage{amsmath}
\usepackage{amssymb}
\usepackage{amsthm}
\usepackage{multirow}

\newtheorem{theorem}{Theorem}
\newtheorem{lemma}{Lemma}

\newcommand{\std}[1]{{\small$\pm$#1}}

\title{HyperFix: Combinatorial Nonlinear Correction for Task Vector Merging}
\author{
    Hyo Seo Kim,
    Ren Wang\corresponding
}
\affiliations{
    Illinois Institute of Technology\\
    hkim155@hawk.illinoistech.edu, rwang74@illinoistech.edu
}

\nocopyright

\begin{document}

\maketitle

\begin{abstract}
    Task vectors enable model merging without joint retraining. In practice, the subset of task vectors to be merged may vary, but many existing methods use scalar tuning for a particular subset, requiring repeated tuning across subsets and restricting task vector merging to linear rescaling. We therefore formulate merging across varying task subsets as a \emph{combinatorial correction} problem and introduce HyperFix, a lightweight hypernetwork that predicts subset-conditioned nonlinear corrections in weight space. Trained once on singleton, pair, and triple subsets from a task bank, HyperFix generalizes to larger subsets without per-subset optimization. Our local perturbation analysis bounds the residual correction beyond linear merging and motivates learning it from small task updates. Experiments across diverse benchmarks show that HyperFix outperforms existing task vector merging methods while reducing tuning cost.
\end{abstract}

\section{Introduction}
\label{sec:intro}

Task vectors represent the parameter differences between pretrained and task-specific fine-tuned weights. Task arithmetic~\cite{ilharcoediting} linearly combines these task vectors to integrate task-specific capabilities without expensive joint retraining. This approach has demonstrated strong performance across vision~\cite{ortiz2023task}, language~\cite{zhang2025variational, zeng2025robustmerge}, and multimodal domains~\cite{huang2024multimodal}.

Despite these advances, many existing task vector merging methods are developed and evaluated for a single fixed task subset, with scalar tuning for that subset. However, different downstream needs may require different subsets of task-specific capabilities, highlighting the need for merging methods that can accommodate varying target subsets. In particular, the tuned scalar cannot be reliably reused because it is specific to a particular subset, as illustrated in Fig.~\ref{fig:intro_motivation}(a). As diverse subsets need to be considered, repeated per-subset tuning is required, resulting in a combinatorial tuning cost. Moreover, these merging methods are restricted to linear scaling of the merged task vector, limiting their ability to capture nonlinear interactions among tasks.

\begin{figure}[t]
\centering
\includegraphics[width=\linewidth]{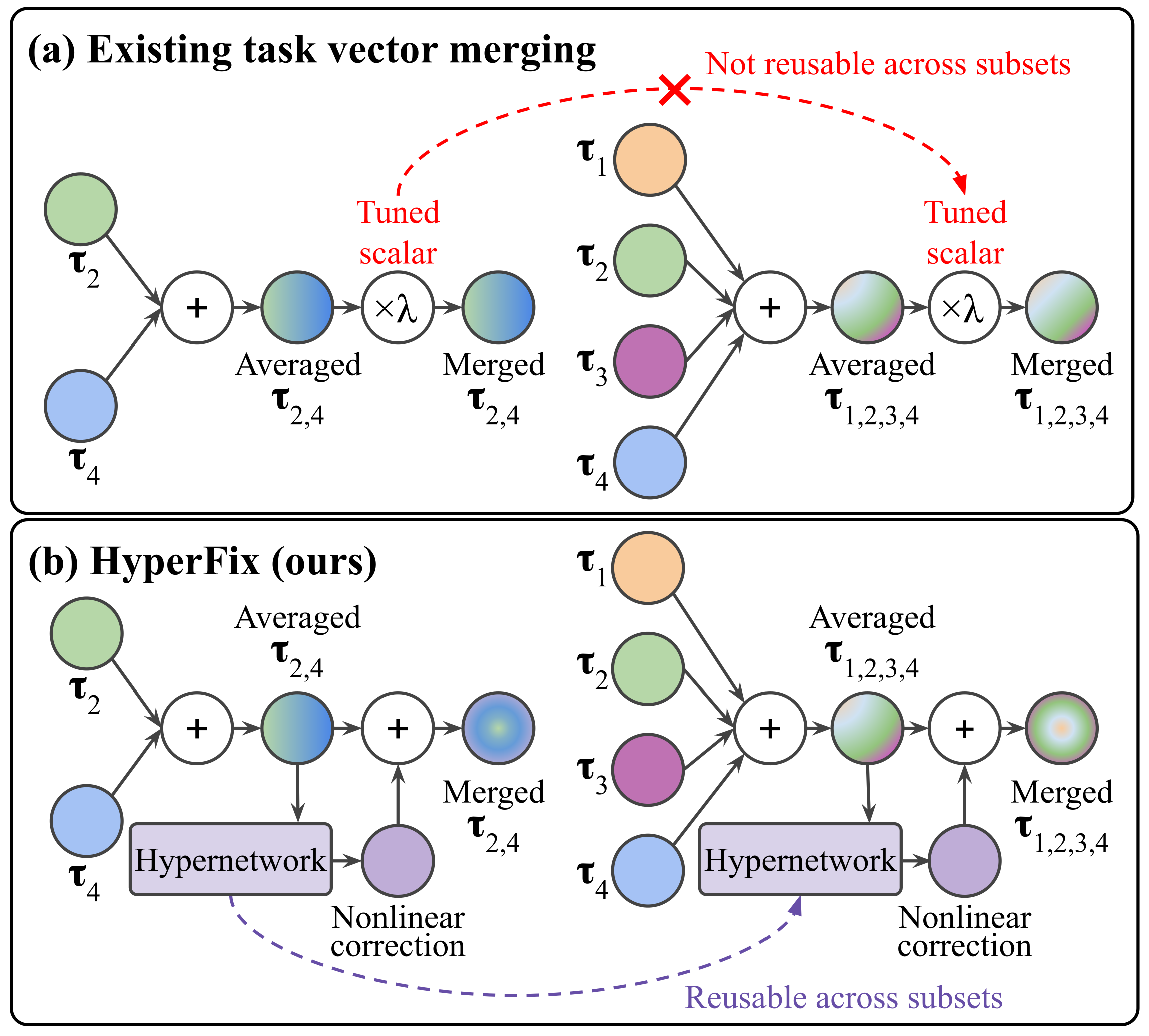}
\caption{
Comparison of the existing merging approach and HyperFix. (a) Existing merging requires separate scalar tuning for each task subset and can only rescale the linearly merged task vector. (b) HyperFix uses a hypernetwork to predict subset-conditioned correction, avoiding per-subset tuning and enabling nonlinear correction.
}
\label{fig:intro_motivation}
\end{figure} 

We therefore view task vector merging as a \emph{combinatorial correction} problem, where the goal is to develop a shared mechanism that constructs a subset-conditioned merged model for any selected subset of task vectors, rather than optimizing a model for a single fixed subset. To address this problem, we introduce HyperFix, a framework that predicts subset-conditioned nonlinear corrections, as illustrated in Fig.~\ref{fig:intro_motivation}(b). Specifically, HyperFix uses a lightweight hypernetwork, trained on small subsets and applied to larger subsets, to generate the corrections.

Our theoretical analysis explains why corrections learned from small subsets can generalize to larger ones. Under mild local smoothness conditions, nonlinear interaction effects can be approximated using low-order subsets. Consistent with this analysis, our experimental results demonstrate that HyperFix consistently outperforms merging baselines across all subset sizes on eight image classification benchmarks and three CLIP~\cite{radford2021learning} architectures. By avoiding repeated per-subset scalar tuning, HyperFix reduces tuning cost by up to 82.6\%.

Our main contributions are as follows:

\begin{itemize}

\item We view task vector merging across varying subsets as a \emph{combinatorial correction} problem and introduce HyperFix, a hypernetwork-based framework that learns a shared mapping from task-subset embeddings to weight corrections that capture nonlinear task interactions.

\item We theoretically show that nonlinear interactions can be approximated using low-order subsets, explaining generalization from small to larger subsets. 

\item Across diverse benchmarks, HyperFix outperforms merging baselines across all subset sizes while reducing tuning cost dramatically.

\end{itemize}

\section{Related Work}

\subsection{Model Merging} 
Model merging aims to combine multiple models fine-tuned from a shared foundation model into a single model that preserves their task-specific capabilities. A fundamental approach is Task Arithmetic~\cite{ilharcoediting}, which represents task-specific adaptations as task vectors and merges them through linear addition. Subsequent methods improve task vector merging through masking and magnitude-based selection~\cite{yadav2023ties, yu2024language, kim2025negmerge}, geometry-aware constraints~\cite{sun2025task, porrello2026dataless}, variational formulations~\cite{zhang2025variational}, and low-rank decomposition of task matrices~\cite{gargiulo2025task}. Adaptive coefficient-based methods such as AdaMerging~\cite{yang2024adamerging} optimize linear merging coefficients for a given set of target tasks. In PEFT settings, alignment-based merging of LoRA adapters has also been explored~\cite{panariello2025accurate}. Despite these advances, most existing methods construct a merged model separately for each target task subset and do not learn a shared merging function across subsets. In contrast, HyperFix learns a shared merging function that predicts subset-conditioned nonlinear corrections.

\subsection{Hypernetwork-based Model Editing} Hypernetworks~\cite{ha2017hypernetworks, krueger2017bayesian, pmlr-v97-ratzlaff19a} are secondary neural networks that generate or modulate the parameters of a target network. In model editing, hypernetworks produce parameter updates conditioned on editing requests. MeG~\cite{wan2025massive} uses a diffusion-based hypernetwork to generate dynamic weight neurons, enabling large-scale knowledge updates without directly modifying internal model weights, unlike approaches such as MEMIT~\cite{mengmass} and MALMEN~\cite{tanmassive}. LoRA.rar~\cite{shenaj2025lora} uses a hypernetwork to predict merging coefficients for combining subject and style LoRAs in real time, reducing computation compared with optimization-based methods. Unlike these hypernetwork-based approaches, HyperFix neither generates full model weights nor predicts scalar merging coefficients. Instead, it predicts a structured low-rank nonlinear correction on top of the linearly merged task vector, conditioned on a permutation-invariant representation of the selected task subset.

\section{Method}
\label{sec:method}

\begin{figure*}[tb]
  \centering  \includegraphics[width=\linewidth]{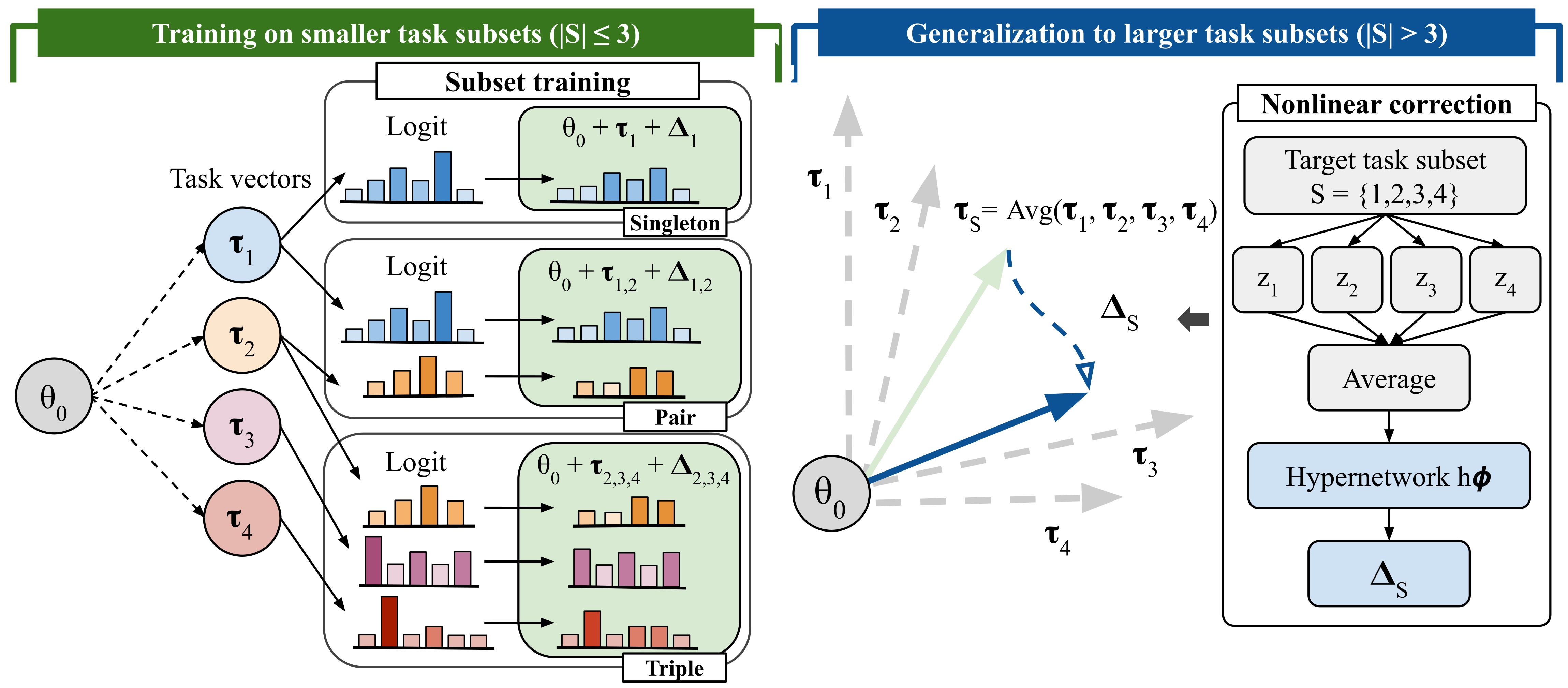}
  \caption{Overview of HyperFix. Given a target task subset $S$, HyperFix constructs the merged task vector $\tau_S$ and uses the subset embedding $z_S$ to predict a subset-conditioned correction $\Delta_S$, yielding $\theta_S=\theta_0+\tau_S+\Delta_S$. Left: HyperFix is trained on task subsets with $|S|\leq3$ (singleton, pair, triple) using KL-based distillation from the corresponding single-task teachers. Right: HyperFix generalizes to larger task subsets ($|S|>3$), where the learned correction is applied without additional optimization.}
  \label{fig:overview}
\end{figure*}

In this section, we present HyperFix, a hypernetwork-based correction framework for task vector merging (Figure~\ref{fig:overview}). We view merging across varying task subsets as a \emph{combinatorial correction} problem, where the goal is to learn a shared correction function. HyperFix
augments linear merging with a subset-conditioned correction generated by a nonlinear hypernetwork. Learned from small task subsets, the shared correction function is applied without additional optimization to larger subsets from the
same task bank. We first motivate the need for subset-conditioned corrections and then describe the subset embedding, hypernetwork parameterization, training objective, and inference procedure.

\subsection{Motivation}

Linear merging assumes that linearly combining task-specific parameter updates in weight space is sufficient to retain the capabilities of the selected tasks in a single model. Let $\theta_0$ denote the pretrained model parameters and $\theta_i$ the fine-tuned parameters for task $i$, and define the task vector as $\tau_i = \theta_i - \theta_0$. For a subset of tasks $S$, linear merging produces a merged task vector $\tau_S$
\begin{equation}
\label{eq:mean}
\tau_S=\frac{1}{|S|}\sum_{i\in S} \tau_i.
\end{equation}
This yields the merged parameters $\theta_0+\tau_S$.
HyperFix operates on top of this merged task vector and can
also be applied to other merging rules, such as TIES~\cite{yadav2023ties}.

However, this additive assumption ignores potential nonlinear interactions between task vectors. In practice, merging multiple task vectors can introduce non-additive interference, leading to performance degradation as the subset size increases. To capture interactions beyond linear merging, we introduce a subset-conditioned correction term 
\begin{equation}
\label{eq:merge}
\theta_S = \theta_0 + \tau_S + \Delta_S,
\end{equation}
where $\Delta_S$ captures structured interaction effects that depend on the specific combination of tasks in $S$. Importantly, these effects are not universal offsets but vary with the alignment and conflict structure among task vectors in the subset. The construction of $\Delta_S$ is detailed in the following subsection.

\subsection{HyperFix: Combinatorial Nonlinear Correction}
\label{sec:hyperfix}
To model interaction effects beyond linear merging, we require a compact representation of a task subset. Conditioning directly on raw task vectors would entail operating in the full encoder parameter space, which contains tens of millions of dimensions (e.g., CLIP ViT-B/32), making learning computationally expensive and difficult to optimize. Instead, we construct a low-dimensional task-level representation that captures relationships between task vectors.

\noindent\textbf{Subset embedding.} We summarize pairwise relationships between task vectors using a task-level Gram matrix
\begin{equation}
\label{eq:gram}
G_{ij}
=
\langle \tau_i , \tau_j \rangle,
\end{equation}
where the inner product is taken over all encoder parameters. The Gram matrix is computed once from the task vectors, mean-centered across tasks, and kept fixed during HyperFix training.

Let $N$ be the number of tasks and $G \in \mathbb{R}^{N \times N}$ the corresponding Gram matrix. We define the task embedding $z_i \in \mathbb{R}^N$ as the $i$-th row of $G$
\begin{equation}
z_i = G_{i:}.
\end{equation}
For a subset $S$, we define the subset embedding
\begin{equation}
\label{eq:embedding}
z_S = \frac{1}{|S|}\sum_{i\in S} z_i.
\end{equation}
The Gram matrix provides pairwise relationship features among task vectors in weight space. Averaging the corresponding task-level embeddings yields a permutation-invariant representation of the interaction structure within a subset. This embedding, therefore, encodes interaction statistics rather than raw parameters, allowing the hypernetwork to focus on modeling residual nonlinear effects.

\noindent\textbf{Structured low-rank correction via hypernetwork.} Given the subset embedding $z_S$, HyperFix uses a hypernetwork $h_\phi$ to predict a structured low-rank correction
\begin{equation}
\label{eq:hyper}
\Delta_S = h_\phi(z_S).
\end{equation}
For each encoder weight matrix $W_\ell \in \mathbb{R}^{d_\ell \times m_\ell}$, HyperFix predicts low-rank factors $U_\ell \in \mathbb{R}^{d_\ell \times r}$ and $V_\ell \in \mathbb{R}^{m_\ell \times r}$, forming a LoRA-style update $\Delta W_\ell = U_\ell V_\ell^\top$, where the rank is set to $r=4$ in all experiments. The low-rank parameterization improves structural efficiency while allowing expressive corrections distributed across layers. The hypernetwork is implemented as a two-layer MLP with a hidden dimension of 512 and GELU activation.  The MLP outputs a single concatenated vector that is partitioned and reshaped into the per-layer factors $(U_\ell, V_\ell)$, thereby producing all encoder updates in a single forward pass. The resulting correction augments the linear merge in Eq.~(\ref{eq:merge}).

\noindent\textbf{Training objective.} Our primary training objective is to optimize the hypernetwork parameters by aligning the predictive distributions of the merged model $\theta_S$ with those of the corresponding single-task models. Specifically, we employ KL-based knowledge distillation. During training, the image encoder and all task-specific classification heads are frozen, and only the HyperFix parameters are optimized. We train HyperFix on subsets of size one to three (singleton, pair, and triple) to learn correction patterns while keeping training computationally tractable.

For each task $i \in S$, we use the corresponding single-task model $\theta_i$ as a reference and compare its predictions with those of the merged model $\theta_S$. 
Let $p_i(\cdot \mid x)$ and $q_i(\cdot \mid x)$ denote the predictive distributions induced by $\theta_i$ and $\theta_S$, respectively, for task $i$ and input $x$. 
We minimize the average KL divergence
\begin{equation}
L_S =
\frac{1}{|S|}
\sum_{i \in S}
\mathrm{KL}\!\left(
p_i(\cdot \mid x)
\;\|\;
q_i(\cdot \mid x)
\right).
\end{equation}
For each $i \in S$, the KL divergence is computed on mini-batches sampled from task $i$'s dataset, and the loss is averaged across tasks in $S$. 

\noindent\textbf{Inference.} At inference time, no additional optimization is required. Given a subset $S$, we compute $z_S$, predict $\Delta_S$ via Eq.~(\ref{eq:hyper}), and construct $\theta_S$ using Eq.~(\ref{eq:merge}). The merged model is then directly evaluated.

\section{Theoretical Analysis}
\label{sec:theory}

In this section, we provide a theoretical analysis of HyperFix, explaining why training on task subsets of size at most three is sufficient for generalization to larger subsets not used during training. \footnote{Our theoretical results are local and perturbative in nature. We analyze behavior in a neighborhood of the pretrained model $\theta_0$ under small task updates, where empirical loss landscapes are often locally smooth and well-conditioned in practice. The assumptions below are not intended to globally characterize deep networks, but to isolate the mechanisms governing nonlinear interaction effects near $\theta_0$.} Detailed proofs are provided in the Supplementary.

\subsection{Preliminaries and Notation}

We adopt the notation introduced in the Method section. Recall that linear merging produces parameters 
$\theta_0 + \tau_S$ (Eq.~\ref{eq:mean}), and HyperFix augments this with a correction term 
as in Eq.~\ref{eq:merge}. For analysis, define the subset-averaged loss
\begin{equation}
L_S(\theta) := \frac{1}{|S|}\sum_{i\in S} L_i(\theta).
\end{equation}
Let $\theta_S^\star$ denote a local stationary point of $L_S$ near $\theta_0$ (i.e., $\nabla L_S(\theta_S^\star)=0$). The \emph{ideal residual correction} beyond linear merging is
\begin{equation}
\Delta^\star_S := \theta_S^\star - (\theta_0+\tau_S).
\end{equation}
Finally, let
\begin{equation}
\bar z := \frac{1}{N}\sum_{i=1}^{N} z_i,
\qquad
\sigma_z^2 := \frac{1}{N}\sum_{i=1}^{N}|z_i-\bar z|^2
\end{equation}
denote the mean and variance of the task embeddings, respectively.

\subsection{Assumptions} We introduce the assumptions used in our analysis.

\noindent\textbf{A1 (Local smoothness).}
Each task loss $L_i(\theta)$ is three-times continuously differentiable in a neighborhood $\mathcal{N}$ of $\theta_0$, and
\begin{equation}
\|\nabla^2 L_i(\theta)\|_{\mathrm{op}} \le H, 
\quad 
\|\nabla^3 L_i(\theta)\|_{\mathrm{op}} \le M,
\end{equation}
for all $\theta \in \mathcal{N}$.

\noindent\textbf{A2 (Local conditioning).}
For any subset $S$, we assume that $L_S$ admits a locally well-conditioned Hessian in a neighborhood of $\theta_0$, i.e., 
\begin{equation}
\nabla^2 L_S(\theta) \succeq \mu I,
\end{equation}
for some $\mu > 0$ in $\mathcal{N}$.

\noindent\textbf{A3 (Small task updates).}
Task vectors satisfy $\|\tau_i\| \le \rho$ for all $i$, where $\rho$ is sufficiently small so that $\theta_0 + \tau_S \in \mathcal{N}$.

\noindent\textbf{A4 (Low-order representability).}
The dominant component of the ideal correction $\Delta^\star_S$ admits a smooth representation in a permutation-invariant subset embedding space
\begin{equation}
\Delta^\star_S = g(z_S) + \eta(S),
\quad
\|\eta(S)\| \le \varepsilon,
\end{equation}
where $g$ is $L_g$-Lipschitz.

\subsection{Low-Order Interaction Generalization}

\begin{theorem}[Nonlinear remainder] 
\label{thm:residual} 

\noindent Under Assumptions A1--A3, 
\begin{equation}
\|\Delta^\star_S\|
\le
\frac{1}{\mu}\bigl(\|\nabla L_S(\theta_0)\|+H\rho\bigr),
\end{equation}
and the nonlinearity beyond the first-order Hessian term satisfies
\begin{equation}
\nabla L_S(\theta_0+\tau_S)
=
\nabla L_S(\theta_0)+\nabla^2 L_S(\theta_0)\tau_S+R_S,
\end{equation}
where $\|R_S\|\le \frac{M}{2}\rho^2$. Hence the parameter effect induced by this nonlinear remainder is bounded by
\begin{equation}
\bigl\| \big(\int_0^1 \nabla^2 L_S(\theta_0+\tau_S+t\Delta^\star_S)\,dt \big)^{-1} R_S \bigr\|
\le
\frac{M}{2\mu}\rho^2.
\end{equation}
\end{theorem}

\subsubsection{Interpretation.}
The first inequality bounds the \emph{overall} correction needed beyond linear merging. It is controlled by (i) local conditioning ($1/\mu$), (ii) the task-update scale $\rho$, and (iii) the gradient bias at $\theta_0$. The gradient term represents baseline mismatch between the pretrained model and the task mixture. It does not scale with task-update magnitude nor reflect higher-order interactions.
The second inequality isolates the \emph{nonlinear interaction} component: under bounded third derivatives,
the part not captured by the first-order Hessian term (instead captured by the curvature-driven remainder) shrinks quadratically with $\rho$ ($O(\rho^2)$). Thus, when task vectors are small, the curvature-induced component of the correction remains controlled. This result motivates learning a compact residual correction.

Theorem 1 establishes that the residual correction beyond linear merging is controlled in magnitude and dominated by low-order interaction terms under small task updates. In particular, the nonlinear remainder scales quadratically with the update size. This suggests that the dominant interaction structure is smooth and locally well-behaved in the subset embedding space. We now formalize how learning such low-order interactions from small task subsets supports generalization to larger subsets not used during training.

\begin{theorem}[Low-order interaction generalization]
\label{thm:low_order}

\noindent Let $h_\phi$ be a hypernetwork trained on task subsets with $|S| \le 3$ to minimize $\mathbb{E}_{|S| \le 3} \| h_\phi(z_S) - \Delta^\star_S \|^2$. Then for any subset size $m \ge 2$,
\begin{equation}
\mathbb{E}_{|S|=m}
\| h_\phi(z_S) - \Delta^\star_S \|
\le \varepsilon
+
\frac{L_g \sigma_z}{\sqrt{m}},
\end{equation}
where $\sigma_z^2$ bounds the variance of the task embeddings $z_i$ across tasks. The $\frac{1}{\sqrt{m}}$ term arises from the concentration of the empirical subset embedding $z_S$ around its expectation as the subset size increases.
\end{theorem}

\subsubsection{Interpretation.}
HyperFix generalizes from small task subsets because task interactions are locally smooth and dominated by low-order effects. If each task update $\tau_i$ is small, and the loss landscape around the pretrained model $\theta_0$ is smooth, then the deviation from linear merging arises primarily from pairwise and triple interactions. Higher-order interactions decay rapidly with the magnitude of task updates. Furthermore, the subset embedding $z_S$ is an empirical average of task interaction statistics. As the subset size increases, this average becomes more stable due to the concentration of measure. Since HyperFix learns a smooth mapping from $z_S$ to correction parameters, it can extrapolate from small task subsets ($|S| \le 3$) to larger task subsets. This explains why training up to triples is sufficient.

\section{Experiments}

In this section, we evaluate HyperFix under the combinatorial correction setting through three key questions. First, can a correction function trained only on small task subsets generalize to larger subsets while maintaining performance? Second, do the predicted corrections meaningfully adapt to the selected task subset, rather than acting as a fixed global adjustment? Third, can HyperFix replace repeated per-subset tuning with a shared training procedure, thereby reducing computational cost while maintaining competitive or superior performance?

\subsection{Setup}

\subsubsection{Combinatorial correction setting.} We evaluate the \emph{combinatorial correction} problem on eight image classification benchmarks: Cars~\cite{krause20133d}, DTD~\cite{cimpoi2014describing}, EuroSAT~\cite{helber2019eurosat}, GTSRB~\cite{stallkamp2011german}, MNIST~\cite{lecun1998mnist}, RESISC45~\cite{cheng2017remote}, SUN397~\cite{xiao2016sun}, and SVHN~\cite{yuval2011reading}. Given task vectors $\{\tau_i\}_{i=1}^N$, the goal is to learn a shared correction function such that, for any subset $S \subseteq \{1,\dots,N\}$, the merged model performs well on all tasks in $S$. HyperFix is trained only on subsets with $|S|\le3$ (singleton, pair, and triple) and evaluated on both subset sizes used during training ($|S|=2,3$) and larger subset sizes not used during training ($|S|=4,\dots,8$). For each subset size $|S|$, we evaluate all $\binom{8}{|S|}$ possible subsets and report their average performance.

\subsubsection{Baselines.}
We evaluate both weight-space~\cite{ilharcoediting} and tangent-space~\cite{ortiz2023task} task vectors with the same merging strategies. \emph{Mean}~\cite{wortsman2022model} performs uniform averaging of task vectors without additional scaling. \emph{Sum $+$ Scalar}~\cite{ilharcoediting} sums task vectors and selects a scalar coefficient separately for each subset from 21 evenly spaced values in $[0,1]$ to maximize the average normalized validation accuracy. \emph{TIES $+$ Scalar}~\cite{yadav2023ties} first resolves sign conflicts using TIES merging and then applies the same scalar search. In contrast, HyperFix augments linear merging with a subset-conditioned nonlinear correction and requires no per-subset tuning. We evaluate \emph{Mean $+$ HyperFix} and \emph{TIES $+$ HyperFix}, where the learned correction is applied on top of the corresponding linear merge rule.

\subsubsection{Training details.}
We conduct experiments using pretrained CLIP models with ViT-B/32, ViT-B/16, and ViT-L/14 backbones~\cite{radford2021learning}. Main results are reported on ViT-B/32, with additional results for ViT-B/16 and ViT-L/14 provided in the Supplementary. During training, all encoder and task-specific classification head parameters are frozen, and only the hypernetwork parameters are optimized. Training is performed for 10 epochs over all singleton, pair, and triple task subsets ($|S|\in\{1,2,3\}$), with random shuffling at each epoch and a batch size of 128 for each task. For each subset, the hypernetwork is optimized by averaging the logit-level KL distillation losses from the corresponding single-task models, using temperature $T=2.0$. We optimize the hypernetwork using AdamW with learning rate $1\times10^{-4}$ and weight decay $0.1$. Gradients are clipped to 1.0. The hypernetwork is a two-layer MLP with hidden dimension 512 and predicts low-rank updates with rank $r=4$ for each encoder weight matrix. Task vectors are computed from independently fine-tuned models on the training split. HyperFix is trained on the validation split, which is also used to tune the scalar coefficients for the scalar-based baselines, and all reported results are evaluated on the held-out test split. All experiments are conducted on a single NVIDIA GH200 GPU.

\subsubsection{Metrics.} We report normalized accuracy following prior work~\cite{ilharcoediting}, defined as the task accuracy divided by the accuracy of the corresponding single-task fine-tuned model. For each subset, normalized accuracy is averaged over all tasks in the subset, and we report the mean and standard deviation across all subsets of the same size. We additionally report absolute accuracies in the Supplementary, where we observe consistent method rankings and performance trends across subset sizes.

\subsection{Generalization Across Task Subsets}

\begin{table*}[t]
\centering

\begin{tabular}{lcccccccc}
\toprule
Method & $|S|=2$ & $|S|=3$ & $|S|=4$ & $|S|=5$ & $|S|=6$ & $|S|=7$ & $|S|=8$ & Avg. \\
\midrule
\multicolumn{9}{l}{Standard Fine-Tuning}\\
\quad Mean
& 95.7\std{1.4} & 89.6\std{2.2} & 84.3\std{2.6} & 80.1\std{2.5} & 77.0\std{2.1} & 74.6\std{1.6} & 72.9\std{0.0} & 82.0 \\
\quad Sum $+$ Scalar
& 96.4\std{1.1} & 92.1\std{1.4} & 88.2\std{1.7} & 84.7\std{1.7} & 81.7\std{1.4} & 79.2\std{1.3} & 77.0\std{0.0} & 85.6 \\
\quad TIES $+$ Scalar
& 97.6\std{1.1} & 94.4\std{1.3} & 91.0\std{1.6} & 87.8\std{1.6} & 85.0\std{1.6} & 82.6\std{1.3} & 80.9\std{0.0} & 88.5 \\

\quad Mean $+$ HyperFix
& 97.7\std{1.7} & 96.3\std{1.8} & 95.1\std{1.6} & 94.0\std{1.4} & 93.2\std{1.1} & 92.6\std{0.8} & 92.0\std{0.0} & 94.4 \\

\quad TIES $+$ HyperFix
& \textbf{98.6\std{1.3}} & \textbf{97.9\std{1.2}}
& \textbf{97.0\std{1.2}} & \textbf{96.0\std{1.1}}
& \textbf{94.9\std{1.0}} & \textbf{93.9\std{0.7}}
& \textbf{92.9\std{0.0}} & \textbf{95.9}\\
\midrule
\multicolumn{9}{l}{Tangent-Space Fine-Tuning}\\
\quad Mean
& 96.8\std{1.4} & 92.5\std{1.9} & 88.8\std{2.1} & 85.8\std{2.0} & 83.4\std{1.7} & 81.4\std{1.3} & 79.8\std{0.0} & 86.9 \\
\quad Sum $+$ Scalar
& 97.8\std{1.3} & 95.4\std{1.5} & 92.9\std{1.5} & 90.8\std{1.3} & 88.8\std{1.1} & 87.2\std{0.9} & 85.7\std{0.0} & 91.2 \\
\quad TIES $+$ Scalar
& 97.5\std{1.4} & 94.5\std{1.5} & 91.9\std{1.5} & 89.7\std{1.3} & 88.0\std{1.1} & 86.4\std{0.9} & 85.0\std{0.0} & 90.4 \\

\quad Mean $+$ HyperFix
& \textbf{99.1\std{1.1}} & 98.1\std{1.0} & 97.0\std{0.9} & 96.0\std{0.8} & 95.0\std{0.7} & 94.1\std{0.6} & 93.2\std{0.0} & 96.1 \\

\quad TIES $+$ HyperFix
& 98.9\std{1.3} & \textbf{98.3\std{1.2}} & \textbf{97.4\std{1.1}} & \textbf{96.6\std{1.0}} & \textbf{95.7\std{0.8}} & \textbf{94.9\std{0.6}} & \textbf{94.1\std{0.0}} & \textbf{96.5} \\
\bottomrule
\end{tabular}
\caption{Generalization performance of HyperFix on task subsets. We evaluate merging performance across subset sizes $|S|=2,\dots,8$ after training only on small task subsets (singleton, pair, and triple). Results are averaged over all $\binom{8}{|S|}$ subsets for each $|S|$ and reported as normalized accuracy (\%). 
 We compare methods using standard fine-tuning~\cite{ilharcoediting} (top) and tangent-space fine-tuning (bottom)~\cite{ortiz2023task}. Avg. reports the average normalized accuracy across all subset sizes. }
\label{tab:merging_results}
\end{table*}

\subsubsection{Training on small task subsets enables generalization to larger subsets.} Table~\ref{tab:merging_results} reports normalized accuracy across all subset sizes. While Mean performs competitively for small task subsets ($|S|=2$), its performance degrades sharply as more tasks are merged, dropping from 95.7 \% to 72.9 \% under standard fine-tuning. Although Scalar tuning mitigates this degradation, it still exhibits a substantial performance drop as $|S|$ increases. In contrast, trained only on task subsets of size at most three, HyperFix maintains high performance on $|S|=4,\dots,8$ without additional tuning. Under standard fine-tuning, Mean $+$ HyperFix improves the $|S|=8$ performance from 72.9 \% (Mean) to 92.0 \%, surpassing both Sum $+$ Scalar (77.0 \%) and TIES $+$ Scalar (80.9 \%), despite these methods performing per-subset scalar tuning (with TIES additionally resolving sign conflicts). TIES $+$ HyperFix further reaches 92.9 \%. Notably, the performance gap between linear merging baselines and HyperFix widens as $|S|$ increases, suggesting that non-additive task interactions become increasingly important for larger task subsets. The same trend holds under tangent-space fine-tuning. Mean $+$ HyperFix consistently outperforms all linear merging baselines across every subset size. This consistency across both fine-tuning regimes indicates that the gains of HyperFix arise from modeling structured task interactions rather than from a specific task-vector construction. We confirm that the performance gains persist across backbone scales, including ViT-B/16 and ViT-L/14 (see Supplementary). The reported standard deviations reflect variation across all $\binom{8}{|S|}$ task subsets. When $|S|=8$, there is only a single subset, resulting in zero standard deviation by definition.

\begin{table*}[t]
\centering

\begin{tabular}{lcccccccc}
\toprule
Objective & $|S|=2$ & $|S|=3$ & $|S|=4$ & $|S|=5$ & $|S|=6$ & $|S|=7$ & $|S|=8$ & Avg. \\
\midrule
CE & 95.0\std{3.6} & 93.6\std{2.9} & 92.1\std{2.5} & 90.8\std{2.1} & 89.6\std{1.7} & 88.7\std{1.2} & 87.8\std{0.0} & 91.1 \\

KL (ours) & \textbf{97.7\std{1.7}} & \textbf{96.3\std{1.8}} & \textbf{95.1\std{1.6}} & \textbf{94.0\std{1.4}} & \textbf{93.2\std{1.1}} & \textbf{92.6\std{0.8}} & \textbf{92.0\std{0.0}} & \textbf{94.4} \\
\bottomrule
\end{tabular}
\caption{Effectiveness of training objective for combinatorial correction. We compare cross-entropy (CE) training with logit-level KL distillation for Mean $+$ HyperFix. We report normalized accuracy (\%) averaged over all task subsets for each subset size. Avg. further averages results across $|S|=2,\dots,8$. KL uses logit-level distillation from the corresponding single-task fine-tuned models, while CE uses ground-truth labels. All other hyperparameters are identical. KL consistently improves performance, particularly for larger task subsets not used during training ($|S|\ge 4$).
}
\label{tab:loss_ablation}
\end{table*}

\subsubsection{KL distillation improves generalization across task subsets.} 
We compare two training objectives for Mean $+$ HyperFix under identical data usage and optimization settings: (i) KL distillation from the corresponding single-task fine-tuned models, and (ii) supervised CE on ground-truth labels from the same training task subsets. Table~\ref{tab:loss_ablation} shows that KL consistently yields stronger performance across all subset sizes, with increasingly larger gains as $|S|$ grows. For example, at $|S|=8$, KL achieves 92.0\% compared to 87.8\% for CE (+4.2 percentage points). Overall, KL improves the average normalized accuracy from 91.1\% to 94.4\% (+3.3 percentage points). These results indicate that distillation provides a smoother and more structured learning signal for predicting parameter corrections than direct supervision with ground-truth labels alone.

\begin{table*}[t]
\centering
\begin{tabular}{lcccccccc}
\toprule
$|S|_{\max}$ & 1 & 2 & 3 & 4 & 5 & 6 & 7 & 8 \\
\midrule
Acc.
& 82.0
& 90.4
& 94.4
& 95.3
& 95.7
& 95.9
& 95.9
& 95.9 \\
Gain
& --
& +8.4
& +4.0
& +0.9
& +0.4
& +0.2
& +0.0
& +0.0 \\
\bottomrule
\end{tabular}
\caption{Ablation on maximum training subset size. HyperFix is trained only on subsets with size $|S|\le |S|_{\max}$, where $|S|_{\max}$ denotes the maximum subset size used during training. Reported values are averaged across subset sizes $|S|=2,\ldots,8$, after averaging over all task subsets at each size. Performance improves as $|S|_{\max}$ increases from 1 to 3, while additional gains beyond $|S|_{\max}\ge4$ are limited, consistent with our theory that dominant interaction effects are captured by low-order subsets.}
\label{tab:training_subset_saturation}
\end{table*}

\subsubsection{Ablation on maximum training subset size.} Table~\ref{tab:training_subset_saturation} analyzes how the maximum subset size $|S|_{\max}$ (i.e., the largest subset size included during HyperFix training) affects merging performance for Mean $+$ HyperFix. Training only on single-task subsets ($|S|_{\max}=1$) yields 82.0\%. Including pairwise subsets ($|S|_{\max}=2$) improves performance to 90.4\% (+8.4 points), indicating that including pair subsets provides substantial additional supervision. Extending training to triple subsets ($|S|_{\max}=3$) further increases performance to 94.4\% (+4.0 points), already achieving strong generalization to unseen larger subsets up to $|S|=8$. Beyond this point, additional exposure to higher-order subsets yields only marginal gains, increasing from 94.4\% to at most 95.9\% (+1.5 points overall). This saturation suggests that training on low-order subsets captures most of the correction patterns needed for larger-subset generalization. These empirical findings are consistent with Theorem~\ref{thm:residual} and Theorem~\ref{thm:low_order}, which together indicate that nonlinear remainders are controlled in magnitude and that once dominant low-order interactions are learned, higher-order effects contribute diminishing additional benefit.

\subsection{Subset-Specificity of Corrections}

\begin{figure}[t]
\centering
    \includegraphics[width=0.9\linewidth]{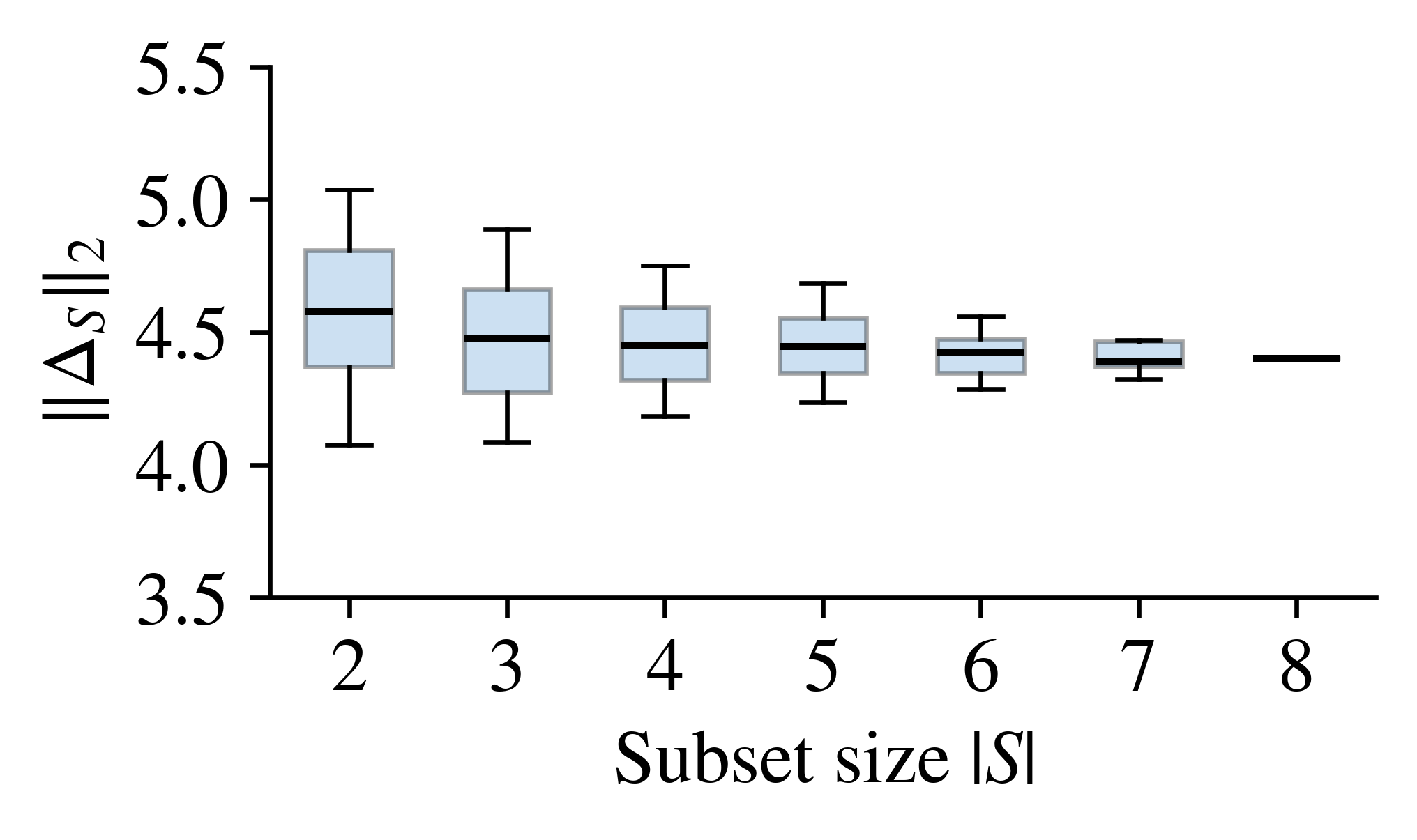}
    \caption{Magnitude of corrections across subset sizes. We show the distribution of the $\ell_2$ norm of the correction $\Delta_S$ for each subset size, computed over all $\binom{8}{|S|}$ task subsets. Boxes indicate interquartile ranges, center lines denote medians, and whiskers show the 5th--95th percentile range. Across all subset sizes, the median remains stable.
    }
    \label{fig:correction_magnitude}
\end{figure}

\begin{table}[t]
    \centering
    \begin{tabular}{l c}
    \toprule
    Method & Acc. (\%) \\
    \midrule
    Mean      & 95.7\std{1.4} \\
    Sign-flipped      & 93.3\std{4.6} \\
    Shuffled  & 95.4\std{3.2} \\
    \textbf{HyperFix (Ours)} & \textbf{97.7\std{1.7}} \\
    \bottomrule
    \end{tabular}
    \caption{Dependence on the subset embedding. We report normalized accuracy (\%) for $|S|=2$ under different inputs. Replacing $z_S$ with shuffled or sign-flipped embeddings degrades performance, indicating that the predicted correction meaningfully depends on the subset representation.}
    \label{tab:input_ablation}
\end{table}

\subsubsection{Magnitude of corrections.} 
To better understand the behavior of the hypernetwork, we analyze the $\ell_2$ norm of the predicted correction $\Delta_S$ across subset sizes $|S|$ for Mean $+$ HyperFix. For each $|S| \in \{2,\dots,8\}$, we compute $\|\Delta_S\|_2$ over all $\binom{8}{|S|}$ task combinations and examine its distribution. As shown in Figure~\ref{fig:correction_magnitude}, the distribution of $\|\Delta_S\|_2$ remains well-controlled across subset sizes. The median remains at a comparable scale across subset sizes. This empirical behavior is consistent with the residual bound discussed in the Theoretical Analysis section, which predicts that the magnitude of the nonlinear correction is controlled by the local update scale under smoothness assumptions.

\subsubsection{Dependence on the subset embedding.} We analyze how the predicted correction depends on the subset embedding $z_S$ under the Mean $+$ HyperFix setting in Table~\ref{tab:input_ablation}. Given the subset embedding $z_S$ for a subset $S$, we consider two perturbations while keeping the base merge $\tau_S$ fixed: (i) \emph{shuffled}: replace $z_S$ with $z_{\tilde S}$ from a randomly sampled $\tilde S \neq S$, and (ii) \emph{sign-flipped}: replace $z_S$ with $-z_S$. Using the correct embedding yields the highest accuracy (97.7\%), improving over the Mean (95.7\%) by +2.0 points. In contrast, perturbing the input embedding degrades performance, sign-flipping reduces accuracy to 93.3\%, and shuffling yields 95.4\%, both close to or below the Mean. These results indicate that the predicted correction meaningfully depends on the subset embedding rather than acting as a fixed offset independent of the selected task subset.

\subsection{Training Cost and Computational Efficiency}

Although scalar-based merging requires no additional training, it relies on validation time optimization. For each task subset $S$, scalar tuning evaluates $K$ candidate coefficients, each requiring evaluation over all tasks in $S$. Aggregated over all task subsets up to a maximum size $|S|_{\max}$, the cumulative validation cost becomes $\sum_{j=2}^{|S|_{\max}} \binom{N}{j} j K$. In our eight task setting with $K=21$, this results in $21{,}336$ validation model evaluations before testing. In contrast, HyperFix shifts this optimization to a single amortized training phase. We train the hypernetwork once, using only task subsets with $|S|\le3$, corresponding to $\sum_{j=1}^{3} \binom{N}{j}$ training subsets, after which the learned correction function is fixed and applied to all subsets without any per-subset optimization. For a fair comparison, HyperFix is trained on the validation split, using the same data employed for scalar coefficient search. As the number of tasks increases, the number of possible task subsets grows combinatorially, whereas the training cost of HyperFix scales only with subsets up to size three. To quantify the efficiency gain, we measure the wall-clock time required to construct and evaluate the merged model under identical hardware and data settings. All experiments are conducted on a single NVIDIA GH200 (120GB HBM3) GPU with CUDA 12.4. For the full task subset with $|S|=8$, scalar tuning requires \textbf{1634.26} seconds due to validation search over $K=21$ candidate coefficients. In contrast, HyperFix requires only \textbf{284.16} seconds without any per-subset search, reducing the total execution time by \textbf{82.6\%}. This empirical gap directly reflects the elimination of repeated validation-time optimization in HyperFix.

\section{Conclusion}

We introduced \emph{combinatorial correction} as the problem of learning a shared correction function over the combinatorial space of task subsets. We showed that linear task vector merging faces both representational and scalability limitations as the subset size increases. Although trained only on small task subsets, HyperFix generalizes to larger subsets not used during training without per-subset optimization, consistently improving performance across diverse benchmarks. These results demonstrate the importance of modeling nonlinear task interactions for scalable and generalizable task vector merging, opening a path toward more principled and general model merging frameworks.

\bibliography{main}

\clearpage
\section{Supplementary Material}

\noindent This supplementary material provides further experimental results and theoretical proof details.

\subsection{Further Experimental Results}
\label{sec:add_exp}

\subsubsection{Results on Additional Backbones.}

To verify that the benefits of HyperFix are not specific to a single backbone scale, we additionally evaluate on CLIP ViT-B/16 and ViT-L/14. Table~\ref{tab:merging_results_backbones} shows that the overall trend observed on ViT-B/32 remains consistent across larger backbones. Linear merging methods degrade as the subset size increases, whereas Mean + HyperFix maintains much stronger and more stable performance across all subset sizes. For example, on ViT-B/16, the performance of Mean drops from 96.4\% at $|S|=2$ to 77.5\% at $|S|=8$, while Mean + HyperFix retains 92.7\% at $|S|=8$ and improves the overall average from 85.0\% to 95.1\%. Similar behavior is observed on ViT-L/14, where HyperFix again achieves the best average performance and remains stable as more tasks are merged. Each configuration is trained once. Reported means and standard deviations are computed across all $\binom{8}{|S|}$ task subsets.

\begin{table*}[b]
\centering
\caption{Generalization performance across additional backbones. Results are averaged over all $\binom{8}{|S|}$ subsets and reported as normalized accuracy (\%). Values after
$\pm$ denote the standard deviation across subsets. Avg. reports
the average normalized accuracy across subset sizes
$|S|=2,\ldots,8$.}
\label{tab:merging_results_backbones}

\begin{tabular}{llcccccccc}
\toprule
Backbone & Method & $|S|=2$ & $|S|=3$ & $|S|=4$ & $|S|=5$ & $|S|=6$ & $|S|=7$ & $|S|=8$ & Avg. \\
\midrule
\multirow{3}{*}{ViT-B/16} & Mean
& 96.4\std{1.6} & 91.2\std{2.5} & 86.7\std{2.8} & 83.3\std{2.6} & 80.9\std{2.1} & 79.0\std{1.5} & 77.5\std{0.0} & 85.0 \\
& Sum + Scalar
& 97.4\std{1.1} & 94.0\std{1.5} & 90.7\std{1.6} & 87.7\std{1.6} & 85.1\std{1.5} & 83.0\std{1.2} & 81.2\std{0.0} & 88.4 \\
& Mean + HyperFix
& \textbf{98.3\std{1.1}} & \textbf{97.0\std{1.4}} & \textbf{95.8\std{1.4}} & \textbf{94.8\std{1.3}} & \textbf{94.0\std{1.1}} & \textbf{93.3\std{0.8}} & \textbf{92.7\std{0.0}} & \textbf{95.1} \\
\midrule
\multirow{3}{*}{ViT-L/14} & Mean
& 98.0\std{1.0} & 94.7\std{1.8} & 91.7\std{1.9} & 89.2\std{1.8} & 87.2\std{1.5} & 85.6\std{1.2} & 84.4\std{0.0} & 90.1 \\
& Sum + Scalar
& 98.8\std{0.8} & 97.3\std{1.0} & 95.6\std{1.0} & 94.0\std{1.1} & 92.3\std{0.9} & 90.8\std{0.7} & 89.7\std{0.0} & 94.1 \\
& Mean + HyperFix
& \textbf{98.9\std{0.7}} & \textbf{97.9\std{0.8}} & \textbf{97.0\std{0.9}} & \textbf{96.1\std{0.8}} & \textbf{95.4\std{0.6}} & \textbf{94.8\std{0.5}} & \textbf{94.3\std{0.0}} & \textbf{96.4} \\
\bottomrule
\end{tabular}
\end{table*}

\begin{table*}[t]
\centering
\caption{Absolute accuracy across task subsets. Results are averaged over all $\binom{8}{|S|}$ subsets and reported as absolute accuracy (\%). Values after
$\pm$ denote the standard deviation across subsets. Avg. reports
the average absolute accuracy across subset sizes
$|S|=2,\ldots,8$.}
\label{tab:merging_results_absolute}

\begin{tabular}{llcccccccc}
\toprule
Backbone & Method & $|S|=2$ & $|S|=3$ & $|S|=4$ & $|S|=5$ & $|S|=6$ & $|S|=7$ & $|S|=8$ & Avg. \\
\midrule
& \multicolumn{9}{l}{Standard Fine-Tuning}\\
\multirow{11}{*}{ViT-B/32} & Mean
& 86.6\std{7.1} & 81.1\std{5.0} & 76.2\std{3.9} & 72.3\std{3.0} & 69.3\std{2.3} & 67.1\std{1.6} & 65.4\std{0.0} & 74.0 \\
& Sum + Scalar
& 87.3\std{7.1} & 83.5\std{4.8} & 80.0\std{3.4} & 76.8\std{2.6} & 74.1\std{2.0} & 71.7\std{1.4} & 69.8\std{0.0} & 77.6 \\
& TIES + Scalar
& 88.2\std{6.9} & 85.4\std{4.9} & 82.4\std{3.6} & 79.5\std{2.8} & 77.0\std{2.2} & 74.8\std{1.7} & 73.1\std{0.0} & 80.1 \\
& Mean + HyperFix
& 88.4\std{7.7} & 87.2\std{6.0} & 86.2\std{4.8} & 85.3\std{3.8} & 84.6\std{2.9} & 84.0\std{2.0} & 83.5\std{0.0} & 85.6 \\
& TIES + HyperFix
& \textbf{89.1\std{7.3}} & \textbf{88.6\std{5.6}}
& \textbf{87.8\std{4.5}} & \textbf{86.9\std{3.6}}
& \textbf{86.0\std{2.8}} & \textbf{85.2\std{1.9}}
& \textbf{84.3\std{0.0}} & \textbf{86.8}\\
\cmidrule(lr){2-10}
& \multicolumn{9}{l}{Tangent-Space Fine-Tuning}\\
& Mean
& 84.5\std{6.6} & 80.9\std{5.4} & 77.7\std{4.5} & 75.1\std{3.6} & 72.9\std{2.8} & 71.2\std{2.0} & 69.7\std{0.0} & 76.0 \\
& Sum + Scalar
& 85.3\std{6.4} & 83.3\std{4.9} & 81.3\std{4.0} & 79.4\std{3.2} & 77.8\std{2.5} & 76.4\std{1.8} & 75.1\std{0.0} & 79.8 \\
& TIES + Scalar
& 85.1\std{6.4} & 82.6\std{4.9} & 80.4\std{3.9} & 78.5\std{3.1} & 77.0\std{2.4} & 76.0\std{1.7} & 74.5\std{0.0} & 79.1 \\
& Mean + HyperFix
& \textbf{86.5\std{6.7}} & 85.6\std{5.1} & 84.7\std{4.0} & 83.9\std{3.2} & 83.0\std{2.5} & 82.2\std{1.7} & 81.5\std{0.0} & 83.9 \\
& TIES + HyperFix
& 86.3\std{6.8} & \textbf{85.8\std{5.2}} & \textbf{85.1\std{4.1}} & \textbf{84.4\std{3.2}} & \textbf{83.7\std{2.4}} & \textbf{83.0\std{1.6}} & \textbf{82.3\std{0.0}} & \textbf{84.4} \\
\midrule
\multirow{3}{*}{ViT-B/16} & Mean
& 89.3\std{6.1} & 84.6\std{5.0} & 80.4\std{4.2} & 77.3\std{3.4} & 74.9\std{2.6} & 73.1\std{1.8} & 71.7\std{0.0} & 78.8 \\
& Sum + Scalar
& 90.2\std{5.9} & 87.2\std{4.3} & 84.2\std{3.4} & 81.4\std{2.7} & 79.1\std{2.2} & 77.1\std{1.5} & 75.5\std{0.0} & 82.1 \\
& Mean + HyperFix
& \textbf{91.0\std{5.9}} & \textbf{89.8\std{4.8}} & \textbf{88.8\std{3.9}} & \textbf{87.9\std{3.2}} & \textbf{87.1\std{2.5}} & \textbf{86.5\std{1.7}} & \textbf{86.0\std{0.0}} & \textbf{88.2} \\
\midrule
\multirow{3}{*}{ViT-L/14} & Mean
& 92.4\std{4.8} & 89.3\std{4.2} & 86.5\std{3.4} & 84.1\std{2.8} & 82.2\std{2.2} & 80.7\std{1.6} & 79.6\std{0.0} & 85.0 \\
& Sum + Scalar
& 93.1\std{4.6} & 91.8\std{3.4} & 90.2\std{2.6} & 88.6\std{2.1} & 87.1\std{1.6} & 85.7\std{1.1} & 84.6\std{0.0} & 88.7 \\
& Mean + HyperFix
& \textbf{93.2\std{4.8}} & \textbf{92.3\std{3.7}} & \textbf{91.4\std{3.0}} & \textbf{90.7\std{2.4}} & \textbf{90.0\std{1.9}} & \textbf{89.4\std{1.3}} & \textbf{89.0\std{0.0}} & \textbf{90.9} \\
\bottomrule
\end{tabular}
\end{table*}

\subsubsection{Absolute Accuracy on Task Subsets.}

While the main paper reports normalized accuracy, we additionally report absolute accuracy. Table~\ref{tab:merging_results_absolute} summarizes the results across all task subsets. The overall trends are consistent with the normalized accuracy results. Linear merging methods degrade substantially as the subset size increases, whereas HyperFix maintains much stronger performance across all subset sizes. For example, under standard fine-tuning on ViT-B/32, Mean decreases from 86.6\% at $|S|=2$ to 65.4\% at $|S|=8$, while Mean + HyperFix retains 83.5\%. TIES + HyperFix further improves this to 84.3\%, consistently outperforming scalar-based baselines. A similar pattern holds under tangent-space fine-tuning, where HyperFix variants achieve the strongest performance across all subset sizes. The larger standard deviations in absolute accuracy mainly arise from differences in the performance of the single-task fine-tuned models across datasets. Since normalized accuracy is divided by the corresponding single-task performance, it reduces this cross-dataset variation and exhibits smaller variance. Overall, these results confirm that the improvements of HyperFix persist when evaluated using absolute accuracy.

\begin{table*}[t]
\centering
\caption{Full results for maximum training subset size. HyperFix is trained on subsets with size $|S|\le |S|_{\max}$. Reported values are the average normalized accuracy (\%) across all $\binom{8}{|S|}$ task subsets.}
\label{tab:training_subset_saturation_full}
\small
\begin{tabular}{lcccccccc}
\toprule
$|S|_{\max}$ & $|S|=2$ & $|S|=3$ & $|S|=4$ & $|S|=5$ & $|S|=6$ & $|S|=7$ & $|S|=8$ & Avg. \\
\midrule
1
& 95.7\std{1.4} & 89.6\std{2.2} & 84.3\std{2.6} & 80.1\std{2.5} & 77.0\std{2.1} & 74.6\std{1.5} & 72.8\std{0.0} & 82.0 \\
2
& 98.0\std{0.7} & 95.3\std{1.2} & 92.4\std{1.4} & 89.8\std{1.4} & 87.5\std{1.2} & 85.6\std{0.9} & 84.0\std{0.0} & 90.4 \\
3
& 97.7\std{1.7} & 96.3\std{1.8} & 95.1\std{1.6} & 94.0\std{1.4} & 93.2\std{1.1} & 92.6\std{0.8} & 92.0\std{0.0} & 94.4 \\
4
& 97.6\std{1.7} & 96.5\std{1.8} & 95.7\std{1.7} & 95.0\std{1.4} & 94.5\std{1.2} & 94.2\std{0.8} & 93.8\std{0.0} & 95.3 \\
5
& 98.1\std{1.3} & 96.9\std{1.5} & 96.0\std{1.5} & 95.3\std{1.3} & 94.9\std{2.2} & 94.6\std{0.8} & 94.3\std{0.0} & 95.7 \\
6
& 97.6\std{2.2} & 97.6\std{2.2} & 96.0\std{1.7} & 95.5\std{1.4} & 95.1\std{1.1} & 94.8\std{0.8} & 94.6\std{0.0} & 95.9 \\
7
& 97.9\std{1.5} & 96.9\std{1.6} & 96.1\std{1.5} & 95.6\std{1.3} & 95.2\std{1.1} & 94.9\std{0.8} & 94.7\std{0.0} & 95.9 \\
8
& 97.8\std{1.8} & 96.8\std{1.7} & 96.1\std{1.6} & 95.6\std{1.3} & 95.2\std{1.9} & 94.9\std{0.8} & 94.7\std{0.0} & 95.9 \\
\bottomrule
\end{tabular}
\end{table*}

\subsubsection{Full Results for the Training Subset Size Ablation.}

In the main paper, we report the average normalized accuracy across all subset sizes to analyze the effect of the maximum training subset size $|S|_{\max}$. For completeness, Table~\ref{tab:training_subset_saturation_full} provides the full performance across subset sizes. Consistent with the trends in the main paper, increasing the training subset size from $|S|_{\max}=1$ to $|S|_{\max}=3$ leads to substantial improvements across most subset sizes. For example, the average accuracy increases from 82.0\% to 94.4\%. This suggests that learning pairwise and triple task interactions plays an important role in modeling nonlinear task vector merging. Beyond $|S|_{\max}\ge4$, additional gains become relatively small. This saturation indicates that a large part of the interaction effects can already be captured by pairwise and triple task interactions.

\subsection{Theory Proof Details}
\label{sec:theory_proof}

\subsubsection{Proof of Theorem 1: A bound on the residual correction magnitude}

\begin{lemma}[Residual correction is controlled by local geometry]
\label{lem:residual_bound}
Under the assumptions in the main paper, the ideal correction $\Delta^\star_S$ satisfies
\begin{equation}
\|\Delta^\star_S\|
\;\le\;
\frac{1}{\mu}\,\|\nabla L_S(\theta_0+\tau_S)\|.
\end{equation}
Moreover, local smoothness gives
\begin{equation}
\|\nabla L_S(\theta_0+\tau_S)\|
\le
\|\nabla L_S(\theta_0)\| + H\|\tau_S\|.
\end{equation}
Since $\|\tau_S\|\le\rho$, we further have
\begin{equation}
\|\nabla L_S(\theta_0)\| + H\|\tau_S\|
\le
\|\nabla L_S(\theta_0)\| + H\rho.
\end{equation}
In the common case where $\|\nabla L_S(\theta_0)\|$ is small (e.g., $\theta_0$ is a good shared initializer),
this yields $\|\Delta^\star_S\| = O(\rho)$.
\end{lemma}

\begin{proof}
By definition,
$\nabla L_S(\theta_S^\star)=0$ and
$\theta_S^\star=\theta_0+\tau_S+\Delta_S^\star$.
Define the averaged Hessian along the segment from
$\theta_0+\tau_S$ to $\theta_S^\star$ by
\begin{equation}
A_S
:=
\int_0^1
\nabla^2 L_S
\bigl(\theta_0+\tau_S+t\Delta_S^\star\bigr)\,dt.
\end{equation}
The mean-value form of Taylor's theorem then gives
\begin{equation}
\nabla L_S(\theta_S^\star)
=
\nabla L_S(\theta_0+\tau_S)
+
A_S\Delta_S^\star.
\end{equation}
Since $\nabla L_S(\theta_S^\star)=0$, it follows that
\begin{equation}
A_S\Delta_S^\star
=
-\nabla L_S(\theta_0+\tau_S).
\end{equation}
By local stability, each Hessian inside the integral has a minimum eigenvalue at least $\mu$,
hence the averaged Hessian is also $\succeq \mu I$ and is invertible with operator norm
at most $1/\mu$. Therefore,
\begin{equation}
\|\Delta^\star_S\|
\le
\frac{1}{\mu}\|\nabla L_S(\theta_0+\tau_S)\|.
\end{equation}
For the second inequality, apply Taylor's theorem to $\nabla L_S$ at $\theta_0$:
\begin{equation}
\begin{aligned}
\nabla L_S(\theta_0+\tau_S)
&=
\nabla L_S(\theta_0) \\
&\quad+
\left(
\int_0^1
\nabla^2 L_S(\theta_0+t\tau_S)\,dt
\right)\tau_S.
\end{aligned}
\end{equation}
Taking norms and using $\|\nabla^2 L_S(\cdot)\|_{\mathrm{op}}\le H$ (since it averages task Hessians bounded by $H$),
\begin{equation}
\|\nabla L_S(\theta_0+\tau_S)\|
\le
\|\nabla L_S(\theta_0)\| + H\|\tau_S\|.
\end{equation}
Finally, $\|\tau_S\|\le \frac{1}{|S|}\sum_{i\in S}\|\tau_i\|\le \rho$.
\end{proof}

\subsubsection{Proof of Theorem 1: Why higher-order interactions decay as $O(\rho^2)$}

\begin{lemma}[Third-order Taylor remainder bound]
\label{lem:taylor_remainder}
Under the smoothness assumptions, for any $u$ with $\|u\|\le 2\rho$,
\begin{equation}
\nabla L_S(\theta_0+u)
=
\nabla L_S(\theta_0)
+\nabla^2 L_S(\theta_0)\,u
+R_S(u),
\end{equation}
where the remainder satisfies
\begin{equation}
\|R_S(u)\|\le \frac{M}{2}\|u\|^2.
\end{equation}
\end{lemma}

\begin{proof}
This is the standard integral remainder form for the gradient:
\begin{equation}
\begin{aligned}
\nabla L_S(\theta_0+u)
&=
\nabla L_S(\theta_0)+\nabla^2 L_S(\theta_0)u \\
&\hspace{-2.5em}+
\int_0^1
\Bigl[\nabla^2 L_S(\theta_0+tu)-\nabla^2 L_S(\theta_0)\Bigr]u\,dt.
\end{aligned}
\end{equation}
Using the mean value form for Hessians with bounded third derivative,
$\|\nabla^2 L_S(\theta_0+tu)-\nabla^2 L_S(\theta_0)\|_{\mathrm{op}}\le M\,t\|u\|$,
we obtain
\begin{equation}
\|R_S(u)\|
\le
\int_0^1 M\,t\|u\|\cdot \|u\|\,dt
=
\frac{M}{2}\|u\|^2.
\end{equation}
\end{proof}

\begin{lemma}[Residual contains only low-order effects up to $O(\rho^2)$]
\label{lem:rho2}
Assume additionally that $\|\nabla L_S(\theta_0)\|$ is small (or treat it as a constant bias term).
Let $u=\tau_S$.
Then
\begin{equation}
\begin{gathered}
\|\nabla L_S(\theta_0+u)\| \\
\le \|\nabla L_S(\theta_0)\|
+\|\nabla^2 L_S(\theta_0)\|_{\mathrm{op}}\|u\|
+\frac{M}{2}\|u\|^2 \\
\le
\|\nabla L_S(\theta_0)\|
+H\rho
+\frac{M}{2}\rho^2.
\end{gathered}
\end{equation}
Consequently, by Lemma~\ref{lem:residual_bound},
\begin{equation}
\bigl\| \big(\int_0^1 \nabla^2 L_S(\theta_0+\tau_S+t\Delta^\star_S)\,dt \big)^{-1} R_S \bigr\|
\le
\frac{M}{2\mu}\rho^2.
\end{equation}
\end{lemma}

\begin{proof}
Apply Lemma~\ref{lem:taylor_remainder} with $u=\tau_S$ and $\|u\|\le\rho$.
Then combine with Lemma~\ref{lem:residual_bound}.
\end{proof}

\paragraph{Interpretation.}
The above shows (conservatively) that when task vectors are small, and the loss is smooth, the needed correction is \emph{dominated by low-order Taylor terms}.
Any interaction effects that require higher-order derivatives become rapidly smaller with $\rho$.

\subsubsection{Proof of Theorem 2: Why a hypernetwork conditioned on $z_S$ can generalize}

HyperFix predicts $\Delta^{\mathrm{hyper}}_S=h_\phi(z_S)$ with
$z_S = \frac{1}{|S|}\sum_{i\in S} z_i$.
The key property is that $z_S$ is \emph{size-invariant} and \emph{averages} per-task interaction descriptors.

\begin{lemma}[Stability of the set embedding under increasing subset size]
\label{lem:embedding_concentration}
Assume the vectors $z_i$ have bounded second moment under the (empirical) task distribution:
$\mathbb{E}\|z_i-\mathbb{E}z\|^2 \le \sigma_z^2$.
If a subset $S$ of size $m$ is sampled by drawing tasks approximately i.i.d.\ from this distribution, then
\begin{equation}
\mathbb{E}\|z_S-\mathbb{E}z\| \le \frac{\sigma_z}{\sqrt{m}}.
\end{equation}
\end{lemma}

\begin{proof}
Since $z_S$ is the average of $m$ samples, its variance scales as $1/m$:
$\mathbb{E}\|z_S-\mathbb{E}z\|^2 \le \sigma_z^2/m$.
Jensen's inequality gives $\mathbb{E}\|z_S-\mathbb{E}z\| \le \sqrt{\mathbb{E}\|z_S-\mathbb{E}z\|^2}
\le \sigma_z/\sqrt{m}$.
\end{proof}

\begin{lemma}[Why larger subsets are not harder in embedding space]
\label{lem:why_generalize}
Under Assumption~A4, if $g$ is $L_g$-Lipschitz, then
\begin{equation}
\mathbb{E}\|\Delta^\star_S - g(\mathbb{E}z)\|
\le
\varepsilon + \frac{L_g\sigma_z}{\sqrt{m}}.
\end{equation}
\end{lemma}

\begin{proof}
By the triangle inequality and Assumption~A4,
\begin{equation}
\begin{gathered}
\|\Delta^\star_S-g(\mathbb{E}z)\| \\[1pt]
{}\le
\|g(z_S)-g(\mathbb{E}z)\|
+\|\eta(S)\| \\[1pt]
{}\le
L_g\|z_S-\mathbb{E}z\|
+\varepsilon.
\end{gathered}
\end{equation}
Taking expectations and applying Lemma~\ref{lem:embedding_concentration} yields the result.
\end{proof}

\end{document}